\documentclass{amsart}
\usepackage[margin=1in]{geometry}
\usepackage[foot]{amsaddr}

\usepackage[utf8]{inputenc}
\usepackage[T1]{fontenc}
\usepackage[dvipsnames,svgnames]{xcolor}
\usepackage{graphicx, float, subcaption, wrapfig, verbatim}
\usepackage{amsmath, amssymb, amsfonts, amsthm, mathtools}
\usepackage{booktabs}
\usepackage{nicefrac}
\usepackage{microtype}
\usepackage{hyperref}
\usepackage[export]{adjustbox}
\usepackage{indentfirst}
\usepackage[page,header]{appendix}
\usepackage{titletoc}
\usepackage{pseudo}

\usepackage[maxbibnames=99]{biblatex}
\mathtoolsset{showonlyrefs}
\renewcommand{\d}{\,\mathrm{d}}
\newcommand{\R}{\mathbb{R}}
\newcommand{\E}{\mathbb{E}}
\newcommand{\N}{\mathcal{N}}

\newcommand{\score}{\nabla \log}
\newcommand{\F}[2]{\ensuremath{\operatorname{F}\!\left(#1\,\|\,#2\right)}}
\newcommand{\Fn}[2]{\ensuremath{\operatorname{F}^n\!\left(#1\,\|\,#2\right)}}
\newcommand{\KL}[2]{\ensuremath{\operatorname{KL}\!\left(#1\,\|\,#2\right)}}

\newcommand{\angles}[2]{\left\langle #1, #2 \right\rangle}
\newcommand{\intr}{\int_{\R^d}}

\newcommand{\ddtau}[1][]{\frac{\partial #1}{\partial \tau}}
\newcommand{\at}[2]{\left.#1\right|_{#2}}
\newcommand{\ddtauatt}[1][]{\at{\ddtau[#1]}{\tau=t}}

\newcommand{\f}{f}

\numberwithin{equation}{section}
\newtheorem{theorem}{Theorem}[section]

\newtheorem{remark}[theorem]{Remark}

\begin{document}

\title[Score-based deterministic density sampling]{Score-based deterministic density sampling}

\author[V.~Ilin]{Vasily Ilin$^{1*}$}
\address{$^1$Department of Mathematics, University of Washington, Seattle, WA 98195, United States}
\email{vilin@uw.edu}

\author[P.~Sushko]{Peter Sushko$^2$}
\address{$^2$Allen Institute for AI, Seattle, WA 98103, United States}
\email{peters@allenai.org}

\author[J.~Hu]{Jingwei Hu$^3$}
\address{$^3$Department of Applied Mathematics, University of Washington, Seattle, WA 98195, United States}
\email{hujw@uw.edu}

\thanks{$^*$Corresponding author.}

\begin{abstract}
We propose a deterministic sampling framework using Score-Based Transport Modeling for sampling an unnormalized target density $\pi$ given only its score $\nabla \log \pi$. Our method approximates the Wasserstein gradient flow on $\mathrm{KL}(f_t\|\pi)$ by learning the time-varying score $\nabla \log f_t$ on the fly using score matching. While having the same marginal distribution as Langevin dynamics, our method produces smooth deterministic trajectories, resulting in monotone noise-free convergence. We prove that our method dissipates relative entropy at the same rate as the exact gradient flow, provided sufficient training. Numerical experiments validate our theoretical findings: our method converges at the optimal rate, has smooth trajectories, and is often more sample efficient than its stochastic counterpart. Experiments on high-dimensional image data show that our method produces high-quality generations in as few as 15 steps and exhibits natural exploratory behavior. The memory and runtime scale linearly in the sample size.
\end{abstract}


\maketitle

\section{Introduction}
Diffusion generative modeling (DGM) \cite{song2019generative, song2020score} has emerged as a powerful set of techniques to generate ``more of the same thing'', i.e., given many samples from some unknown distribution $\pi$, train a model to generate more samples from $\pi$. While the SDE-based generation is commonly used in practice, its deterministic counterpart, termed ``probability flow ODE'' by \cite{song2020score}, offers several practical advantages -- higher order solvers \cite{huang2024convergence}, better dimension dependence \cite{chen2024probability}, and interpolation in the latent space \cite{song2020denoising}. The two processes are given \cite{chen2022sampling} by the reverse of the Ornstein-Uhlenbeck (OU) process and the corresponding ODE, respectively:
\begin{align}
    \d X_t &= \left(X_t + 2\score f_t(X_t)\right)\d t + \sqrt{2}\d B_t \label{eqn: DGM SDE} \tag{DGM SDE}\\
    \d X_t &= \left(X_t + \score f_t(X_t)\right)\d t,\label{eqn: DGM ODE} \tag{DGM ODE}
\end{align}
where $B_t$ is the Brownian motion and $f_t$ is the law of the OU process. The score $\score f_t$ is learned by running the OU process from $\pi$ to $\N(0,I_d)$, which crucially depends on having many samples from $\pi$. In this work we pose and attempt to answer the following question: \textbf{How to deterministically sample an unnormalized density $\pi$ in the absence of samples, using techniques of diffusion generative modeling?} We propose an analogue of \eqref{eqn: DGM ODE} when no samples from $\pi$ are available. The absence of samples from $\pi$ makes sampling a much harder problem than DGM. Indeed, DGM can be reduced to unnormalized density sampling by estimating the score $\score \pi$ on the samples, but the converse is not true; recently \cite{he2025query} proved the exponential in $d$ query complexity lower bound for density sampling, while DGM is solvable in polynomial time \cite{chen2022sampling}.

{\bf Contributions.} Our main contributions are a deterministic sampling algorithm, scalable to high dimensions, and a proof of fast convergence under the log-Sobolev inequality. Unlike the classical Langevin dynamics, our algorithm produces smooth deterministic trajectories, and gives access to the otherwise intractable score $\score f_t$ allowing for convergence estimation via the relative Fisher information. Unlike other deterministic interacting particle systems, the memory and runtime scale linearly in the number of particles. The convergence analysis is based on a system of coupled gradient flows, one for the samples and one for the weights of the neural network. This is a novel framework that can be of independent interest as a general technique for analyzing convergence of NN-based approximations of gradient flows. We are not aware of any other work that uses the neural tangent kernel for analyzing dynamically changing loss. Our main theoretical result is Theorem \ref{thm: main result}.

{\bf Related work.} Sampling from unnormalized densities traditionally relies on stochastic algorithms such as Langevin dynamics, whose fast convergence under isoperimetry assumptions is well-understood \cite{dalalyan2017user,wibisono2018sampling,vempala2019rapid,chewi2023log}. Deterministic alternatives, notably Stein Variational Gradient Descent (SVGD) \cite{liu2016stein,duncan2019geometry}, rely heavily on handcrafted kernels, resulting in limited theoretical guarantees and empirical success, especially in high dimensions \cite{zhuo2018message}. Inspired by deterministic probability-flow ODEs from diffusion generative modeling \cite{hyvarinen2005scorematching,song2019generative,song2020score,chen2024probability,huang2024convergence}, which often outperform their stochastic counterparts due to high-order integrators, better dimension dependence, and latent-space interpolation capabilities \cite{song2020denoising,lu2022dpm}, Boffi and Vanden-Eijnden proposed \cite{boffi2023probability} Score-Based Transport Modeling (SBTM) as a general method of solving the Fokker-Planck equation with a deterministic interacting particle system, building on neural approximations of gradient flows \cite{xu2022normalizing,elamvazhuthi2024score}. However, it has not been used for sampling. Our work fills this gap by adapting the method for the problem of sampling and by providing entropy dissipation guarantees and proving rapid convergence for the coupled particle system dynamics, integrating smoothly with recent annealing approaches like those proposed by \cite{neal2001annealed,corrales2025annealed,chehab2024practical}.

{\bf Organization.} The rest of the paper is organized as follows. In Section \ref{sec: GF} we introduce the main algorithm, building on Score-Based Transport Modeling \cite{boffi2023probability}. In Section \ref{sec: convergence} we prove that the resulting dynamics achieve the optimal convergence rate. In Section \ref{sec: experiments} we verify convergence in several numerical experiments. The paper is concluded in Section \ref{sec:conclusion}.

\begin{figure}[htp!]
    \centering
    \includegraphics[width=\linewidth]{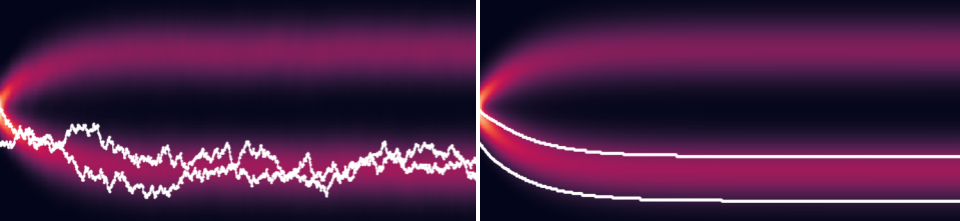}
    \caption{Left: Langevin dynamics (stochastic). Right: ours (deterministic). The deterministic algorithm has the same marginal distributions as the stochastic one but with smooth trajectories. In this plot both algorithms interpolate between the unit Gaussian and a mixture of two Gaussians.}
    \label{fig: SDE vs SBTM bimodal gaussian}
\end{figure}

\section{Deterministic sampling algorithm}\label{sec: GF}
Diffusion generative modeling \eqref{eqn: DGM SDE} and \eqref{eqn: DGM ODE} converges quickly because it is the reverse of a fast process, namely the OU process. In the absence of samples from $\pi$, there is no known numerically tractable process that can be run from $\pi$ to $f_0$ and reversed, and the dimension-exponential lower bound of \cite{he2025query} suggests that it cannot exist. Thus, the standard approach to classical sampling is to greedily minimize a divergence, most commonly relative entropy $\KL{f_t}{\pi} := \E_{f_t} \log\frac{f_t}{\pi}$, between $f_t$ and $\pi$ at every time step. In continuous time, this is called a Wasserstein gradient flow (GF). The Wasserstein GF on $\KL{\cdot}{\pi}$ can be implemented stochastically or deterministically, mimicking equations \eqref{eqn: DGM SDE} and \eqref{eqn: DGM ODE}:
\begin{align}
    \d X_t &= \score \pi(X_t) \d t + \sqrt{2} \d B_t,\quad B_t := \text{Brownian motion}, \label{eqn: GF SDE} \tag{GF SDE}\\
    \d X_t &= \score \pi(X_t) \d t - \score f_t(X_t) \d t, \quad f_t := \text{law}(X_t). \label{eqn: GF ODE} \tag{GF ODE}
\end{align}

\begin{figure}
\centering
\includegraphics[width=0.5\linewidth]{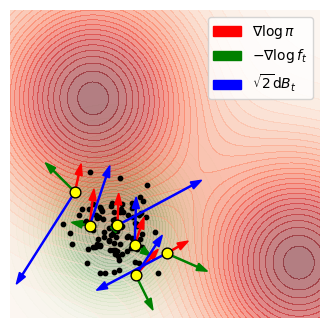}
\caption{Visualization of \eqref{eqn: GF ODE}. Particles are pulled towards target \textcolor{red}{$\pi$} and away from particle density \textcolor{DarkGreen}{$f_t$}. \textcolor{blue}{Brownian motion} acts randomly.}
\label{fig: score visualization}
\end{figure}
Remarkably, the distributions of $X_t$ for a fixed $t$ in equations \eqref{eqn: GF SDE} and \eqref{eqn: GF ODE} match exactly. See \cite{jordan1998variational} for the rigorous derivation and connection to Optimal Transport and see Figure \ref{fig: SDE vs SBTM bimodal gaussian} for a visualization. While the former equation is well-studied under the names Langevin Dynamics or ULA, the latter can be more appealing due to determinism, for example due to smooth trajectories, higher order time integrators, and potentially better sample complexity. In this work we simulate the \eqref{eqn: GF ODE}. The term $-\score f_t$ in the ODE provides deterministic exploration by making samples $X_t^1,\dots,X_t^n$ repel from each other, since more particles appear where $f_t$ is larger. We visualize the forces $\score \pi, -\score f_t$ and the noise $\sqrt{2}\d B_t$ in Figure \ref{fig: score visualization}.

\begin{remark}
    While equations \eqref{eqn: DGM ODE} and \eqref{eqn: GF ODE} look similar, there is a big difference in how $f_t$ is defined. In \eqref{eqn: DGM ODE}, $f_t$ is given by the OU process started from $\pi$. In \eqref{eqn: GF ODE} $f_t$ follows the gradient flow started from $f_0$, usually chosen to be Gaussian $\N(0,1)$.
\end{remark}

\subsection{Score-Based Transport Modeling}

Score-based transport modeling (SBTM) was introduced in \cite{boffi2023probability} as a method of solving Fokker-Planck equations, of which \eqref{eqn: GF ODE} is a special case. Additionally, SBTM was recently successfully employed to solve other Fokker-Planck-type equations \cite{ilin2024transport, lu2024score, huang2024score}.

The core idea of SBTM is to approximate $\score f_t$ with a neural network $s^{\Theta_t}$ trained on the sample $X_t^1, ..., X_t^n$. The training is done by minimizing the score matching loss from \cite{hyvarinen2005scorematching, song2019generative}. The SBTM dynamics approximate \eqref{eqn: GF ODE} and are given by the interacting particle system
\begin{equation} \label{eqn: SBTM} \tag{SBTM}
    \begin{aligned}
        \frac{\d X_t}{\d t} &= \score \pi(X_t) - s^{\Theta_t}(X_t), \quad X_t = (X_t^1,...,X_t^n), \quad X^i_0 \sim f_0 \text{ iid}\\
        \frac{\d \Theta_t}{\d t} &= -\eta \nabla_{\Theta_t} L(s^{\Theta_t}, f_t), \quad X_t \sim f_t,
    \end{aligned}
\end{equation}
where $L(s, f_t) = \E_{f_t} \|s - \score f_t\|^2$ is the dynamically changing score-matching loss \cite{hyvarinen2005scorematching} and $\eta(t)$ controls the amount of NN training relative to sampling dynamics. The resulting interacting particle system is a coupled gradient flow, where $f_t$ approximates the Wasserstein GF on $\KL{f_t}{\pi}$, and $\Theta_t$ undergoes the Euclidean GF on the dynamically changing loss $L(s^{\Theta_t}, f_t)$, while $\eta$ controls the relative speed of the two flows. If the loss $L(s,f)$ is identically zero, then SBTM dynamics exactly recover the \eqref{eqn: GF ODE} dynamics. Unlike most other interacting particle systems used for sampling \cite{liu2016stein, liu2017stein, corrales2025annealed}, in SBTM the particles $X^1_t,\dots,X^n_t$ interact via the neural network $s^{\Theta_t}$, which makes the memory and runtime scale linearly, and provides much better dimension scaling. 

\begin{remark}[Comparison with Langevin Dynamics]
    In equation \eqref{eqn: GF SDE}, also called Langevin Dynamics, particles $X_t^i$ are iid, whereas in SBTM dynamics the particles interact via the NN $s^{\Theta_t}$. Additionally, while particles in Langevin Dynamics undergo Brownian motion, resulting in continuous but non-differentiable trajectories, particles in SBTM have smooth trajectories. Thus, while the time- and particle-marginals of Langevin and SBTM dynamics coincide in the limit $L(s,f) \to 0$, the joint distributions are very different. 
\end{remark}


While the loss $L(s, f_t) = \E_{f_t} \|s - \score f_t\|^2$ is not tractable due to the score term $\score f_t$, for smooth $f$, it is equivalent to the loss $\E_{f_t} \|s\|^2 + 2\nabla \cdot s$ via integration by parts:
\begin{align}
    \E_f \|s - \score f\|^2 
    &= \E_f \left(\|s\|^2 - 2s \cdot \score f\right) + \E_f \|\score f\|^2 \tag*{\text{(expand $\|\cdot\|^2$)}}\\
    &= \E_f \left(\|s\|^2 + 2\nabla\cdot s\right) + \text{const}(s) \tag*{\text{(integrate by parts)}}.
\end{align}
The latter loss is tractable, and is readily approximated by evaluating on a sample $X^1,\dots,X^n$ from $f$:
\begin{align}
    \E_f \left(\|s\|^2 + 2\nabla\cdot s\right) = \frac{1}{n} \sum_{i=1}^n \left(\|s(X^i)\|^2 + 2\nabla\cdot s(X^i)\right) + \text{const}(s) \quad \text{if} \quad f = \frac{1}{n} \sum_{i=1}^n \delta_{X^i}.
\end{align}
While we use automatic differentiation to compute the divergence $\nabla \cdot s$ in the low-dimensional experiments, we employ Hutchinson's trace estimator \cite{hutchinson1989stochastic} $\E \left[\varepsilon^T J \varepsilon \right]$ where $\varepsilon$ follows the Rademacher distribution. This is because computing $\nabla \cdot s$ with autodiff takes $O(dC_s)$ time, where $C_s$ is the cost of one NN evaluation, whereas the Hutchinson's trace estimator takes $O(m C_s)$, where $m$ is the number of random projections.

\begin{remark}
    The neural network trains on the particles that were produced following the same neural network, as opposed to the true solution of the Wasserstein GF \eqref{eqn: GF ODE}. The commonsense intuition makes one suspect that local errors may accumulate exponentially. Miraculously, this does not happen, see Remark \ref{rem: local error does not accumulate}.
\end{remark}

For ease of reproducibility, we provide the pseudocode for SBTM:
\begin{pseudo}*
    \hd{SBTM}(f_0, \score \pi, n, T, \Delta t, \eta) \label{alg: SBTM}\\
    sample $\{X^i\}_{i=1}^n$ iid from $f_0$ \\
    initialize NN: $s^\Theta \approx \arg\min L(s, f_0)$ \\
    $t := 0$ \\
    while $t < T$\\+
        $t := t + \Delta t$\\
        for $k = 1,\dots,K$ \hspace{2cm} (flow $\Theta$) \\+
            $\Theta = \Theta - \eta\cdot \nabla_\Theta L(s^\Theta, f_t)$\\-
        for $i = 1,\dots, n$ \hspace{2cm} \ \ (flow $f$) \\+
            $X^i := X^i + \Delta t (\score \pi(X^i) - s^\Theta(X^i))$\\--
    output particle locations $X^1,\dots,X^n$
\end{pseudo}
In practice we use the Adamw optimizer for the gradient descent step and use the denoised score matching loss for high-dimensional experiments. To initialize the NN $s^\Theta \approx \arg\min L(s, f_0)$ we perform gradient descent on the standard MSE loss
\begin{align}
    \frac{1}{n} \sum_{i=1}^n \|s(X_0^i) - \score f_0(X_0^i)\|^2,
\end{align}
since we have access to the initial data $f_0$. For example, in most experiments below the initial data is the standard Gaussian, making the initial score $\score f_0(x) = -x$.

\section{Convergence analysis} \label{sec: convergence}

\textbf{How quickly does $f_t$ converge to $\pi$ in \eqref{eqn: SBTM}?} To answer this question we study the decay rate of the relative entropy $\KL{f_t}{\pi}$ \cite{chewi2023log}. Our strategy is to break down the change in $\KL{f_t}{\pi}$ into the change due to the density $f_t$ and the change due to the neural network $s^{\Theta_t}$, effectively splitting the coupled GF into two parts. Theorem \ref{thm: SBTM entropy dissipation} does the former, and Theorem \ref{thm: loss non-increasing} does the latter. Theorem \ref{thm: main result} brings both pieces together. Theorem \ref{thm: entropy dissipation in annealed dynamics} extends Theorem \ref{thm: SBTM entropy dissipation} to the annealed setting.

\subsection{Entropy dissipation}
First, recall the following classical calculation \cite{chewi2023log} that quantifies the rate of relative entropy decay for the Wasserstein GF on relative entropy:
\begin{theorem}\label{thm: entropy dissipation in W2 GF}
    If $f_t$ follows the Wasserstein GF on $\KL{\cdot}{\pi}$ then relative entropy dissipates at the rate of the relative Fisher information:
        \begin{equation}
            -\frac{\d}{\d t}\KL{f_t}{\pi} = \F{f_t}{\pi} := \E_{f_t}\left\|\score \frac{f_t}{\pi}\right\|^2.
        \end{equation}
    Additionally, if $\pi$ satisfies the log-Sobolev inequality with constant $\alpha$ (e.g. if $\pi$ is $\alpha$-log-concave) then
        \begin{equation}
            \KL{f_t}{\pi} \leq \KL{f_0}{\pi} e^{-\frac{1}{\alpha}t}.
        \end{equation}
\end{theorem}
Intuitively, if $s_t \approx \score f_t$, SBTM should converge at the nearly optimal rate $-\frac{d}{dt}\KL{f_t}{\pi} = \F{f_t}{\pi}$. Indeed, this holds if the score matching loss is small:
\begin{theorem}[Small loss guarantees optimal entropy dissipation] \label{thm: SBTM entropy dissipation}
    If $f_t$ is the density of $X_t$, which follows 
    \begin{equation}
        \frac{\d X_t}{\d t} = \score \pi(X_t) - s_t(X_t), \quad X_0 \sim f_0,
    \end{equation}
    for any time-dependent vector field $s_t$, then
    \begin{align}
        -\frac{\d}{\d t}\KL{f_t}{\pi} 
        &=  \F{f_t}{\pi} - \E_{f_t}\angles{s_t - \score f_t}{\score \frac{f_t}{\pi}}\\
        &\geq \frac{1}{2}\F{f_t}{\pi} - \frac{1}{2}L(s_t, f_t). \label{eq: entropy dissipation bound}
    \end{align}
    In particular, if $L(s_t, f_t) \leq \frac{1}{2} \F{f_t}{\pi}$, then
    \begin{equation}
        -\frac{\d}{\d t}\KL{f_t}{\pi} \geq \frac{1}{4}\F{f_t}{\pi}.
    \end{equation}
\end{theorem}
While Theorem \ref{thm: SBTM entropy dissipation} looks simple, it elucidates a remarkable property of SBTM:
\begin{remark}\label{rem: local error does not accumulate}
    Integrating \eqref{eq: entropy dissipation bound} in time, one obtains
    \begin{align}
        \KL{f_T}{\pi} \leq \KL{f_0}{\pi} -\frac{1}{2}\int_0^T \F{f_t}{\pi} \d t + \frac{1}{2} \int_0^T L(s_t, f_t) \d t.
    \end{align}
    Since there no exponential term $e^T$ in the bound, local errors do not accumulate. Importantly, $f_t$ is the law of $X_t$, where $X_t$ follows \eqref{eqn: SBTM}; it is not the (unknown) true solution of \eqref{eqn: GF ODE}.
\end{remark}

\begin{proof}[Proof of Theorem \ref{thm: SBTM entropy dissipation}]
    If $f_t$ is the density of $X_t$, where $X_t$ satisfies
    \begin{align}
        \frac{\d}{\d t} X_t = \score \pi(X_t) - s_t(X_t)
    \end{align}
    then $f_t$ satisfies the Fokker-Planck equation
    \begin{align}
        \partial_t f_t + \nabla \cdot (f_t (\score \pi - s_t)) = 0.
    \end{align}
    Thus, we may explicitly compute the relative entropy dissipation rate as 
    \begin{align}
        &-\frac{\d}{\d t}\intr f_t \log \frac{f_t}{\pi} \d x\\
        &=\intr \angles{s_t - \score \pi}{\score \frac{f_t}{\pi}} f_t \d x + \intr f_t \partial_t \log f_t \d x\\
        &= \intr \angles{s_t - \score \pi}{\score \frac{f_t}{\pi}} f_t \d x,
    \end{align}
    where we used the fact that $\frac{\d}{\d t}\intr f_t \d x = 0$ for the last equality.

    By taking $s_t = \score f_t$ exactly, we recover the proof of the classical result \ref{thm: entropy dissipation in W2 GF}. Otherwise, adding and subtracting $\score f_t$ in the integrand, we get
    \begin{align}
        -&\frac{\d}{\d t}\intr f_t \log \frac{f_t}{\pi} \d x\\
        &= \intr \|\score f_t - \score \pi\|^2 f_t \d x - \intr \angles{\score f_t - s_t}{\score f_t - \score \pi} f_t \d x\\
        &\geq \frac{1}{2}\intr \|\score f_t - \score \pi\|^2 f_t \d x - \frac{1}{2}\intr \|\score f_t - s_t\|^2 f_t \d x\\
        &= \frac{1}{2}\F{f_t}{\pi} - \frac{1}{2}L(s_t,f_t).
    \end{align}
The third line is by Young's inequality.
\end{proof}

\begin{remark}
The Young bound $\tfrac12\F{f_t}{\pi}-\tfrac12 L(s_t,f_t)$ is conservative.  
Using Cauchy--Schwarz instead of Young gives
\[
\Big|-\frac{\d}{\d t}\KL{f_t}{\pi} - \F{f_t}{\pi}\Big|
\le \sqrt{L(s_t,f_t)\,\F{f_t}{\pi}},
\]
so that $-\frac{\d}{\d t}\KL{f_t}{\pi}\to\F{f_t}{\pi}$ at rate $O(\sqrt{L})$ as $L\to 0$. In numerical experiments we observe $-\frac{\d}{\d t}\KL{f_t}{\pi} \approx \F{f_t}{\pi}$, indicating small loss.
\end{remark}

Now we show how to ensure that approximation $s_t \approx \score f_t$ remains accurate despite changing $f_t$. We achieve it by breaking up $\frac{\d}{\d t} L(s_t, f_t)$ into the change due to $s_t$ training and the change due to $f_t$ evolving. Informally speaking, the NN training needs to be faster than the evolution of $f_t$.

\subsection{Bounding score matching loss}
With sufficient size and amount of training a neural network can fit arbitrary data, including dynamically changing data. The next theorem shows that the score-matching loss does not increase, given sufficient training.

\begin{theorem}[Sufficient training guarantees bounded loss]\label{thm: loss non-increasing}
    Assume that $f_t$ is any time-dependent density, $X^1_t,\dots,X^n_t$ are distributed according to $f_t$, and the neural network $s_t = s^{\Theta_t}$ is trained with gradient descent
    \begin{align}
        \frac{\d \Theta_t}{\d t} &= -\eta \nabla_{\Theta_t} L^n(s_t, f_t), \quad L^n(s_t, f_t) := \frac{1}{n} \sum_{i=1}^n \|s(X_t^i) - \score f(X_t^i)\|^2
    \end{align}
    and is such that the Neural Tangent Kernel (NTK)
    \begin{align}
        H_{\alpha, \beta}^{i,j}(t) = \sum_{k=1}^N \nabla_{\theta_k} s_t^\alpha(X_t^i) \nabla_{\theta_k} s_t^\beta(X_t^j), \quad s(x)=(s^1(x),\dots,s^d(x))
    \end{align}    
    is lower bounded by $\lambda=\lambda(t)>0$, i.e. $\|Hv\|^2 \geq \lambda \|v\|^2$. Then as long as
    \begin{align}
        \eta(t) \geq \frac{n}{\lambda} \ddtauatt \log L^n(s_t, \f_\tau),
    \end{align}
    the loss $L^n(s_t, f_t)$ is non-increasing in time.
\end{theorem}

\begin{remark}
    The assumption that the NTK is lower bounded in Theorem \ref{thm: loss non-increasing} is non-trivial, but holds under relatively mild assumptions \cite{karhadkar2024bounds}; the most restrictive one is that the NN size is superlinear in the number of data. While the current state-of-the-art NTK theory cannot handle infinite training data regime $n\to \infty$ with finite NN size, as better NTK theory becomes available, the assumptions in Theorem \ref{thm: loss non-increasing} can be weakened.
\end{remark}

\begin{proof}[Proof of Theorem \ref{thm: loss non-increasing}]
We start with an elementary computation based on chain rule.

\begin{align}
    \frac{\d}{\d t} L(s_t, f_t) &= \ddtauatt L(s_\tau, \f_t) + \ddtauatt L(s_t, \f_\tau)\\
    \ddtauatt L(s_\tau, \f_t)   &= \nabla_\Theta L \cdot \frac{\d}{\d t} \Theta\\
    &= -\eta \sum_{k=1}^N (\nabla_{\theta_k} L)^2\\
    &= -\frac{\eta}{n^2} \sum_{k=1}^N \left( [s(X^i) - \score f(X^i)]\cdot \nabla_{\theta_k}s(X^i) \right)^2\\
    &= -\frac{\eta}{n^2} \sum_{i,j=1}^n \sum_{\alpha,\beta=1}^d [s_\alpha(X^i) - \nabla_\alpha \log f(X^i)] H_{\alpha, \beta}^{i,j} [s_\beta(X^j) - \nabla_\beta \log f(X^j)] \\
    &= -\frac{\eta}{n^2} \|s - \score f\|_H^2.
\end{align}
where
\begin{align}
    H_{\alpha, \beta}^{i,j}=\sum_{k=1}^N \nabla_{\theta_k} s_\alpha(X^i) \nabla_{\theta_k} s_\beta(X^j)
\end{align}
is called the Neural Tangent Kernel. Its lowest eigenvalue determines the convergence speed of gradient descent. If $\|Hv\|^2 \geq \lambda \|v\|^2$, then
\begin{align}
    \frac{\d}{\d t} L(s_t, f_t) 
    &= \ddtauatt L(s_\tau, \f_t) + \ddtauatt L(s_t, \f_\tau) \\
    &\leq -\frac{\eta \lambda}{n} L(s_t, f_t) + \ddtauatt L(s_t, \f_\tau).
\end{align}
Thus, the loss is non-increasing if
\begin{align}
    \eta(t) \geq \frac{n}{\lambda} \ddtauatt \log L(s_t, \f_\tau).
\end{align}
\end{proof}

Combining Theorems \ref{thm: entropy dissipation in W2 GF} and \ref{thm: loss non-increasing} shows that the convergence rate of SBTM matches the convergence rate of the true GF \eqref{eqn: GF ODE}. This is the main theoretical guarantee of our sampling method.
\begin{theorem}\label{thm: main result}
    Suppose that $X_t$ and $\Theta_t$ follow \eqref{eqn: SBTM} and $\eta$ and $H$ are as in Theorem \ref{thm: loss non-increasing}. If the initial loss is small and the true loss $L$ is well-approximated by the training loss $L^n$
    \begin{align}\label{eqn: initial loss is small}
        L^n(s_0, f_0) \leq \frac{1}{4}\varepsilon, \quad \varepsilon := \inf_{t \leq T} \F{f_t}{\pi},\\
        |L^n(s_t, f_t) - L(s_t, f_t)| \leq \frac{1}{4}\varepsilon, \label{eqn: true loss well-approximated by training loss}
    \end{align}
    then relative entropy dissipates at the optimal rate:
    \begin{align}
        -\frac{\d}{\d t}\KL{f_t}{\pi} \geq \frac{1}{4}\F{f_t}{\pi}.
    \end{align}
    If $\pi$ satisfies the log-Sobolev inequality with constant $\alpha$ then convergence is exponential:
        \begin{equation}
            \KL{f_t}{\pi} \leq \KL{f_0}{\pi} e^{-\frac{1}{4\alpha}t}.
        \end{equation}
\end{theorem}

\begin{remark}\label{rem: assumptions of main theorem}
    Since $s_t \approx \score f_t$, relative Fisher information may be approximated by
    \begin{align}
        \F{f_t}{\pi} \approx \Fn{f_t}{\pi} := \frac{1}{n} \sum_{i=1}^n \| s(X^i) - \score \pi(X^i) \|^2.
    \end{align}
    One may stop the sampling when $\Fn{f_t}{\pi} \leq \varepsilon$ for some predefined $\varepsilon$. This makes condition \eqref{eqn: initial loss is small} practical.

    If particles $X^1(t),\dots, X^n(t)$ were independent, the Law of Large Numbers would imply
    \begin{align}
        \lim_{n \to \infty} |L^n(s_t, f_t) - L(s_t, f_t)| = 0,
    \end{align}
    satisfying condition \eqref{eqn: true loss well-approximated by training loss} for large enough $n$. Further, for interacting particle systems propagation of chaos often holds, namely that in the limit $n\to \infty$ particles $X_t^i$ become independent. In numerical experiments, we treat $L^n$ as $L$, and confirm the optimal rate of relative entropy dissipation.
\end{remark}

\subsection{Annealed dynamics}\label{sec: annealing}
Classical sampling can be broken down into three distinct parts: mode discovery, mode weighting, and mode approximation.
While Langevin dynamics performs mode approximation very fast, both mode discovery and mode weighting are challenging. For example, in the absence of samples from $\pi$, the mere discovery of a mode with support of small width $h$ in $d$ dimensions requires $\Omega\left(h^{-d}\right)$ function evaluations \cite{he2025query}. Empirically, it helps to pick an annealing between $f_0$ and $\pi$ to ``guide'' $f_t$, similar to how the forward OU process guides the reverse process in DGM. We use the geometric annealing $\pi_t \propto f_0^{1-t} \pi^t$ \cite{neal2001annealed, mate2023learning, chemseddine2024neural} and the dilation annealing $\pi_t(x) \propto \pi(x/t)$ \cite{chehab2024practical}.


One obtains a similar entropy dissipation estimate in annealed dynamics. Taking $\pi_t = \pi$ recovers Theorem \ref{thm: SBTM entropy dissipation}.
\begin{theorem}\label{thm: entropy dissipation in annealed dynamics}
    If $\pi_t$ is any time-dependent density, $s_t$ any time-dependent vector field, and
    \begin{align}
        \frac{\d X_t}{\d t} &= \score \pi_t(X_t) - s(X_t),
    \end{align}
    then
    \begin{align}
        -\frac{\d}{\d t}\KL{f_t}{\pi} 
        =  \E_{f_t} \angles{s_t - \score \pi_t}{\score \frac{f_t}{\pi}}.
    \end{align}
    In particular, as $L(s_t,f_t) \to 0$,
    \begin{align}\label{eqn: entropy dissipation in annealed dynamics, assuming zero loss}
        -\frac{\d}{\d t}\KL{f_t}{\pi} 
        \to \E_{f_t}\angles{\score \frac{f_t}{\pi_t}}{\score \frac{f_t}{\pi}}
    \end{align}
    at the rate $O(\sqrt{L(s_t,f_t)})$.
\end{theorem}
Theorem \ref{thm: entropy dissipation in annealed dynamics} gives a way to test for $L(s_t,f_t)=0$ by testing equality \eqref{eqn: entropy dissipation in annealed dynamics, assuming zero loss}. This is important, because the true loss $L(s_t, f_t) = \E_f \|s - \score f\|^2$ cannot be computed from a finite sample $X^1_t,\dots,X^n_t$ of $f_t$ since $\score f_t$ is unknown. Empirically, we confirm that \eqref{eqn: entropy dissipation in annealed dynamics, assuming zero loss} holds, indicating small loss -- see Figure \ref{fig: gaussians far 2d annealed entropy dissipation}.

\begin{proof}
    If $f_t$ is the density of $X_t$, where $X_t$ satisfies
    \begin{align}
        \frac{\d}{\d t} X_t = \score \pi_t(X_t) - s_t(X_t),
    \end{align}
    then $f_t$ satisfies the Fokker-Planck equation
    \begin{align}
        \partial_t f_t + \nabla \cdot (f_t (\score \pi_t - s_t)) = 0.
    \end{align}
    Thus, we may explicitly compute the relative entropy dissipation rate as 
    \begin{align}
        &-\frac{\d}{\d t}\intr f_t \log \frac{f_t}{\pi} \d x\\
        &=\intr \angles{s_t - \score \pi_t}{\score \frac{f_t}{\pi}} f_t \d x + \intr f_t \partial_t \log f_t \d x\\
        &= \intr \angles{s_t - \score \pi_t}{\score \frac{f_t}{\pi}} f_t \d x,
    \end{align}
where we used the fact that $\intr f_t \partial_t \log f_t \d x$ is zero for the last equality. Now, 
\begin{align}
    &\left|\intr \angles{s_t - \score \pi_t}{\score \frac{f_t}{\pi}} f_t \d x - \intr \angles{\score \frac{f_t}{\pi_t}}{\score \frac{f_t}{\pi}} f_t \d x\right|^2\\
    &\leq \intr \left|\angles{s_t - \score f_t}{\score \frac{f_t}{\pi}}\right|^2 f_t \d x\\
    &\leq \intr \left| s_t - \score f_t \right|^2 f_t \d x \intr \left| \score f_t - \score \pi \right|^2 f_t \d x \\
    &= L(s_t, f_t) \F{f_t}{\pi}.
\end{align}
Thus, $-\frac{\d}{\d t}\intr f_t \log \frac{f_t}{\pi} \d x \to \E_{f_t} \angles{\score \frac{f_t}{\pi_t}}{\score \frac{f_t}{\pi}}$ as $L \to 0$, at the rate $O(\sqrt{L})$.

\end{proof}

\section{Experiments}\label{sec: experiments}

We present several low- and high-dimensional experiments, comparing SBTM to its stochastic counterpart \eqref{eqn: GF SDE}, and SVGD. The examples differ in the log-concavity of the target, the number of modes, the relative mode weighting, the dimension and the annealing path used. We do not extensively compare SBTM to SVGD \cite{liu2016stein} because without additional tricks such as momentum and cherry-picking the kernel bandwidth we were unable to obtain decent performance from SVGD even in low dimensions, see Tables \ref{tab:kl_analytic} and \ref{tab:kl_gaussians_near}. We emphasize that while the numerical experiments presented below would be trivial in the context of DGM, they are quite challenging in the absence of samples from $\pi$, only using the score $\score \pi$.

In the low-dimensional experiments we confirm that, unlike \eqref{eqn: GF SDE}, the \eqref{eqn: SBTM} dynamics produce smooth trajectories of the relevant metrics: relative entropy, relative entropy dissipation rate, L2 error. Moreover, we verify that the relative entropy dissipation rate almost exactly matches the approximated relative Fisher information, indicating good score approximation and optimal rate of convergence, empirically confirming the validity of Theorems \ref{thm: main result} and \ref{thm: entropy dissipation in annealed dynamics} and estimate \eqref{eqn: entropy dissipation in annealed dynamics, assuming zero loss}, and showcasing the usefulness of having access to the time-varying score $\score f_t$. Finally, the high-dimensional experiment confirms the scaling of SBTM to high dimensions and demonstrates that even in 784 dimensions the score approximation is accurate and SBTM retains exploration. Additionally, SBTM exhibits somewhat better sample efficiency, likely due to particle dependence, and demonstrates early mode splitting as shown in Experiment \ref{exp: Gaussian mixture with dilation annealing}.

We use a three-layer Residual Neural Network of width 128 for 1D experiments, and increase it to five layers for 2D experiments. For the MNIST experiment we use an 8-layer U-Net. For all experiments we use mini-batch gradient descent with the Adamw optimizer. For low-dimensional experiments we use learning rate $5\cdot 10^{-4}$, batch size 400, and 10 gradient descent steps per simulation step. For the MNIST experiment we use learning rate $10^{-3}$ and batch size 64. These hyperameters were found empirically. All experiments were done in JAX on a single TITAN X Pascal 12GB GPU. The runtime of the longest experiment is 90 minutes. The code repository is \href{https://github.com/Vilin97/SBTM-sampling/tree/main}{Vilin97/SBTM-sampling}.

\subsection{Log-concave target}\label{exp: Log-concave target}
\begin{figure}[htp!]
    \centering
    \includegraphics[width=0.6\linewidth]{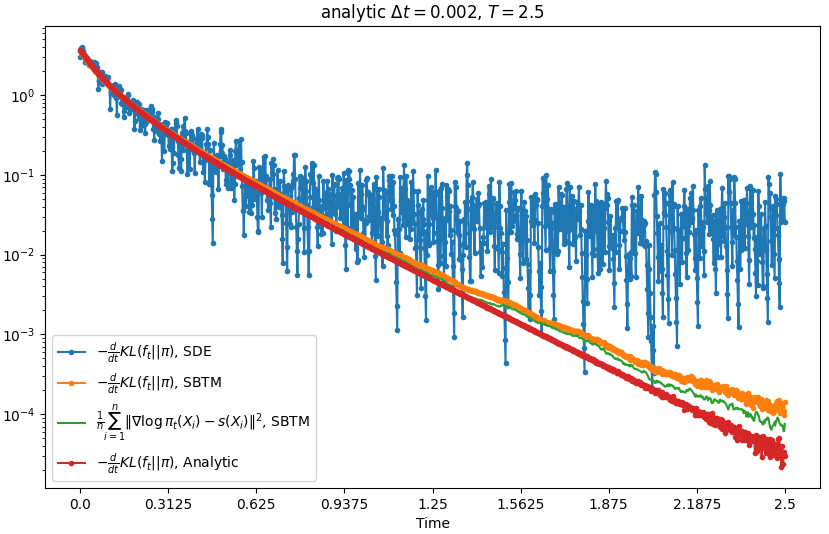}
    \includegraphics[width=0.51\linewidth]{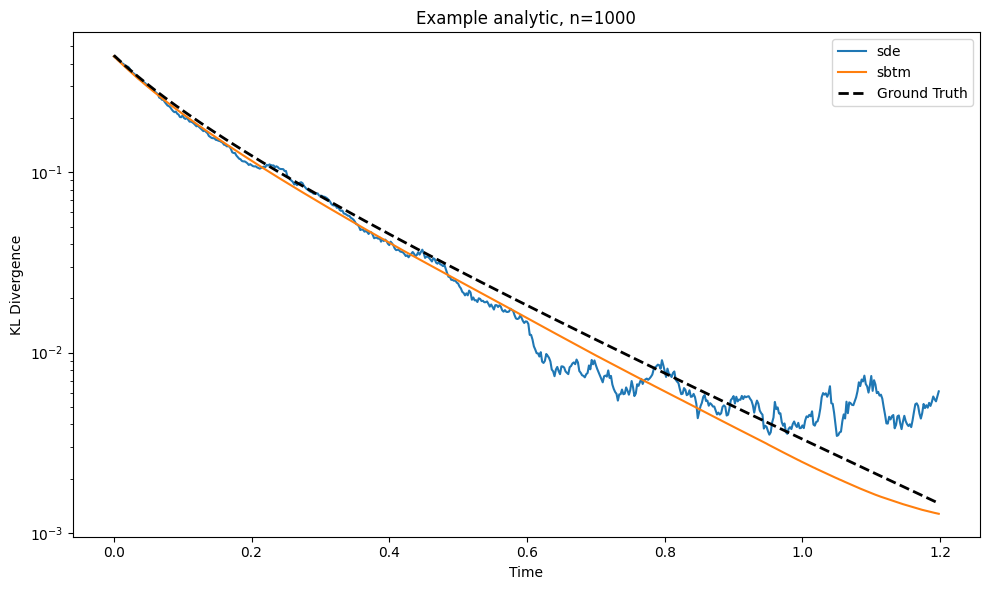}
    \includegraphics[width=0.47\linewidth]{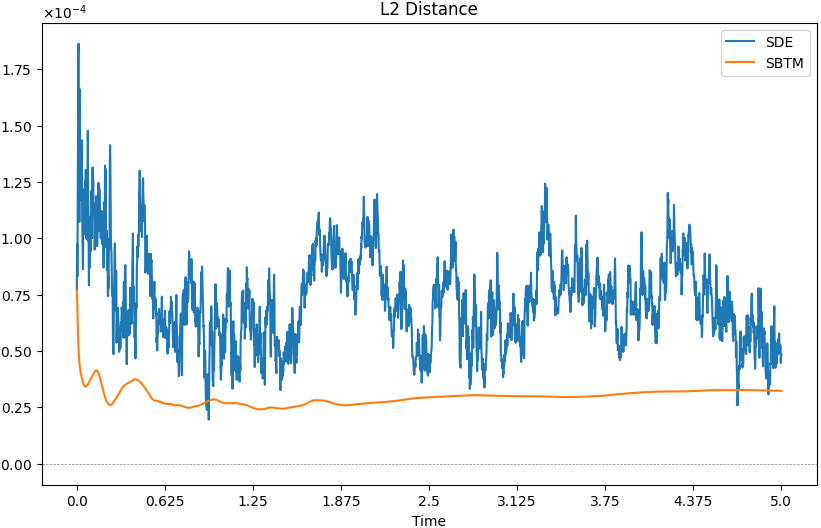}
    \caption{Experiment \ref{exp: Log-concave target}, log-concave target. Top: relative entropy dissipation rate of SBTM (ours) and SDE (stochastic). SBTM approximates entropy decay rate well, while SDE is noisy.
    Bottom left: relative entropy of SBTM, SDE and the ground truth. SBTM approximates the ground truth well.
    Bottom right: L2 error to the true ground truth solution. SBTM produces lower error with smoother trajectory.}
    \label{fig: analytic}
\end{figure}

\begin{wrapfigure}{r}{0.6\textwidth}
\vspace{-2.5em}
\captionsetup{type=table}
\centering
{\footnotesize
\caption{KL divergence ($\downarrow$) between the sample and the target
$\pi = \mathcal{N}(0,1)$, using time step $0.002$ and final time $2.5$. SBTM exhibits better sample efficiency, likely due to determinism.}
\label{tab:kl_analytic}
\begin{tabular}{lccccc}
& \multicolumn{5}{c}{sample size} \\
Method & 100 & 300 & 1000 & 3000 & 10000 \\
\midrule
SBTM (ours) & \textbf{0.013} & \textbf{0.0032} & \textbf{0.0019} & \textbf{0.0020} & \textbf{0.00099} \\
SDE         & 0.020 & 0.022  & 0.0094 & 0.0036 & 0.0012 \\
SVGD        & 0.33  & 0.23   & 0.16   & 0.20   & 0.29   \\
\bottomrule
\end{tabular}
}
\end{wrapfigure}
First, we consider an example that admits an analytic solution $f_t = \N\left(0,1 - e^{-2(t+0.1)}\right)$, which allows us to compute the true entropy dissipation and the $L^2$ distance to the true solution. We use sample size $n=1000$. SBTM exhibits the optimal relative entropy dissipation rate, as evidenced by the close alignment of the orange, green, and red lines in the left panel of Figure \ref{fig: analytic}. The deterministic nature of SBTM allows it to achieve better sample efficiency than the noisy SDE, as shown in Table \ref{tab:kl_analytic}. While the marginal density of $X^i_t$ is the same between SBTM and SDE for a fixed $i$ and $t$, the whole particle ensemble $X^1_t,\dots,X^1_t$ has a different distribution, as the particles in SBTM interact via the NN. More detailed investigation of the ensemble properties of SBTM is left as future work.

\subsection{Gaussian mixture}\label{exp: Gaussian mixture 1D}

\begin{figure}[htp!]
    \centering
    \includegraphics[width=0.49\linewidth]{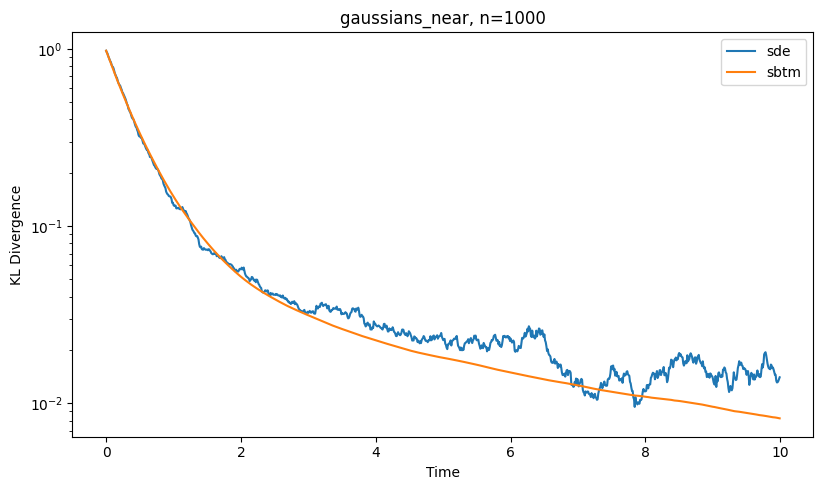}
    \includegraphics[width=0.47\linewidth]{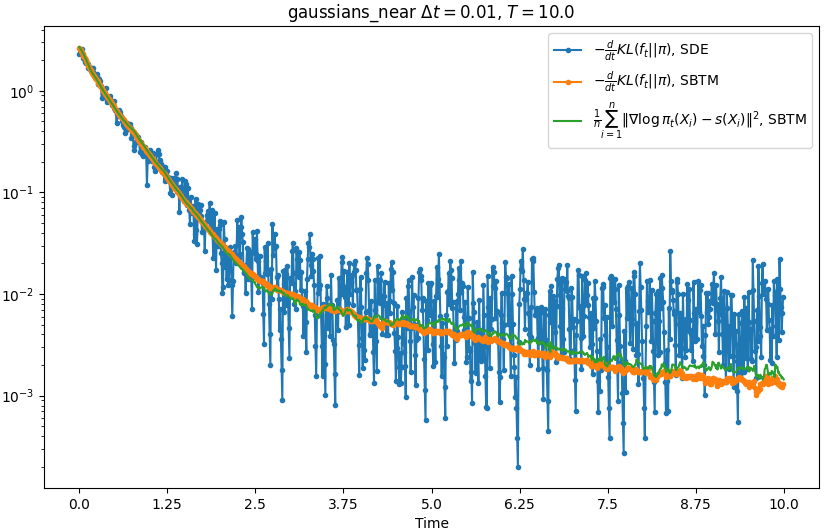}
    \caption{Experiment \ref{exp: Gaussian mixture 1D}, 1D Gaussian mixture. Left: KL divergence of SBTM (ours) and SDE (stochastic) over time. SBTM exhibits smoother convergence. Right: entropy dissipation of SBTM and SDE. SBTM approximates entropy decay rate perfectly with the computable quantity $\F{f_t}{\pi}$, while SDE is noisy.}
    \label{fig: gaussians near 1d}
\end{figure}

\begin{wrapfigure}{r}{0.6\textwidth}
\captionsetup{type=table}
\centering
{\footnotesize
\caption{KL divergence ($\downarrow$) between the sample and the target
$\pi = \frac{1}{4}\mathcal{N}(-2,1) + \frac{3}{4}\mathcal{N}(2,1)$, using time step 0.01, and final time 10. Top: non-annealed. Bottom: 100 particles.}
\label{tab:kl_gaussians_near}
\begin{tabular}{lccccc}
& \multicolumn{5}{c}{sample size} \\
& 100 & 300 & 1000 & 3000 & 10000 \\
\midrule
SBTM (ours) & \textbf{0.022} & 0.018 & \textbf{0.0082} & 0.0082 & \textbf{0.0036} \\
SDE         & 0.029 & \textbf{0.013} & 0.014 & \textbf{0.0068} & 0.0043 \\
SVGD        & 2.8   & 1.4   & 2.4   & 2.1   & 2.0   \\
\bottomrule
\end{tabular}
\vspace{0.6em}
\begin{tabular}{cccc}
& \multicolumn{3}{c}{annealing} \\
& Non-annealed & Geometric & Dilation \\
\midrule
SBTM (ours) & \textbf{0.022} & \textbf{0.058} & \textbf{0.037} \\
SDE & 0.029 & 0.060 & 0.062 \\
SVGD & 2.800 & 0.470 & 0.480 \\
\bottomrule
\end{tabular}
}
\vspace{-1em}
\end{wrapfigure}

This and further examples do not admit analytic solutions, but we can still compare the quality of the final sample as well as compare the trajectories of the SDE and SBTM. We use $f_0 = \mathcal{N}(0,1)$ here and in further experiments. Here we sample from the Gaussian mixture $\pi = \frac{1}{4}\mathcal{N}(-2,1) + \frac{3}{4}\mathcal{N}(2,1)$ with $n=1000$. As evidenced by the change of slope at $t=2.5$ in Figure~\ref{fig: gaussians near 1d}, the underlying Markov chain enters metastability, which plagues the convergence to non-log-concave targets. Here the non-log-concavity is mild, so the process still converges in a reasonable time frame. Table \ref{tab:kl_gaussians_near} shows that SBTM achieves competitive KL divergence.

\subsection{Gaussian mixture with geometric annealing}\label{exp: gaussian mixture geometric annealing}
In this example we sample from the well-separated Gaussian mixture $\pi = \frac{1}{4}\mathcal{N}(-4,1) + \frac{3}{4}\mathcal{N}(4,1)$ with $n=1000$. Because the target is extremely non-log-concave we employ geometric annealing, which linearly interpolates the score between the initial and the target densities: $\score \pi_t = (1-t) \score f_0 + t \score \pi$. Figure \ref{fig: gaussians far 1d} shows that estimate \eqref{eqn: entropy dissipation in annealed dynamics, assuming zero loss} holds empirically almost perfectly, indicating small score-matching loss. The imperfect density reconstruction on the left panel of Figure \ref{fig: gaussians far 1d} is due to the meta-stability issue, which is prevalent even in annealed dynamics; there is no known tractable dynamics without teleportation of mass that perform well on this example.

\begin{figure}[h]
    \centering
    \vspace{0.5em}
    \includegraphics[width=0.49\linewidth]{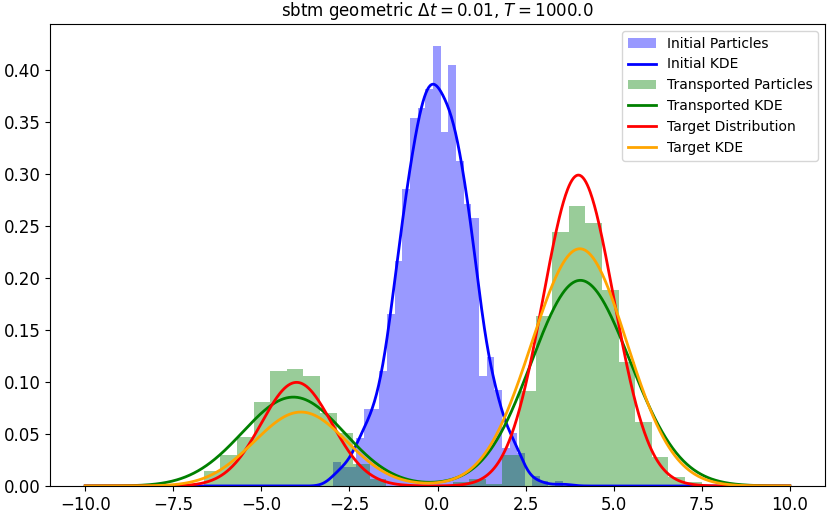}
    \includegraphics[width=0.47\linewidth]{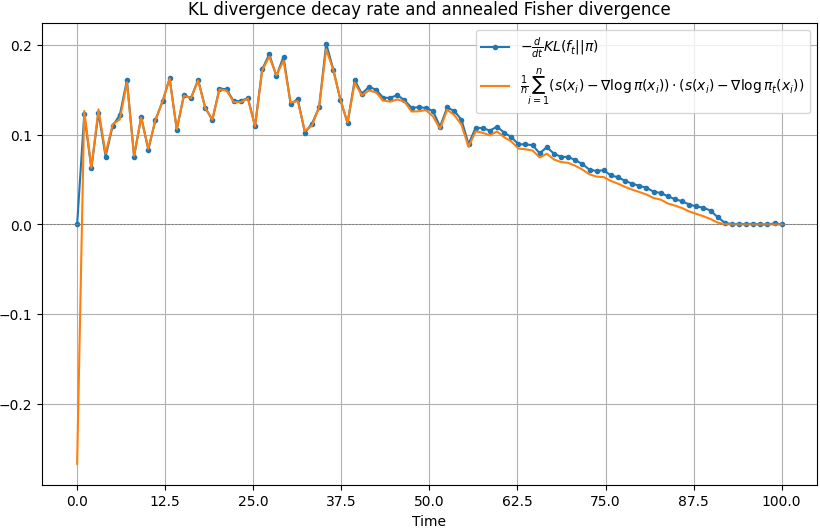}
    \caption{Experiment \ref{exp: gaussian mixture geometric annealing}, well-separated 1D Gaussian mixture. Left: reconstructed density of SBTM. It approximates the solution well despite the non-log-concavity. Right: entropy dissipation of SBTM (ours) and SDE (stochastic). SBTM approximates entropy decay rate perfectly even in annealed dynamics.}
    \label{fig: gaussians far 1d}
\end{figure}

\subsection{Noisy Circle}\label{exp: noisy circle}
\begin{figure}[htp!]
    \centering
    \vspace{1em}
    \includegraphics[width=0.32\linewidth]{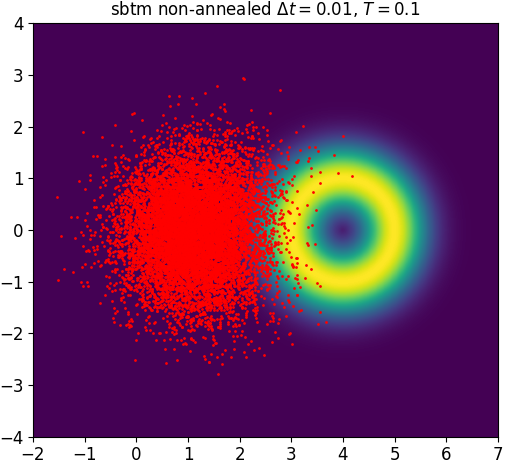}
    \includegraphics[width=0.32\linewidth]{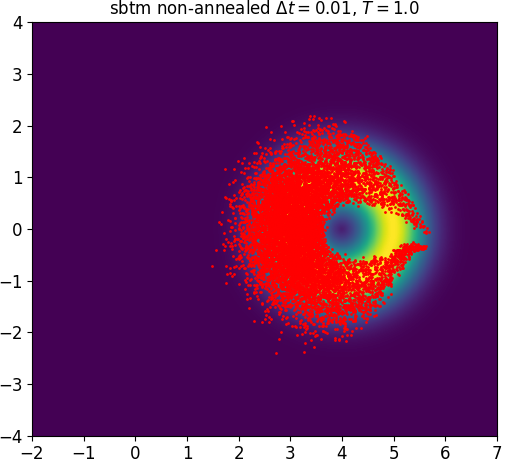}
    \includegraphics[width=0.32\linewidth]{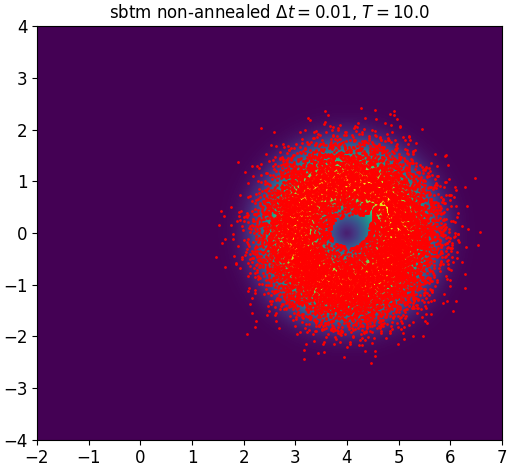}
    \includegraphics[width=0.32\linewidth]{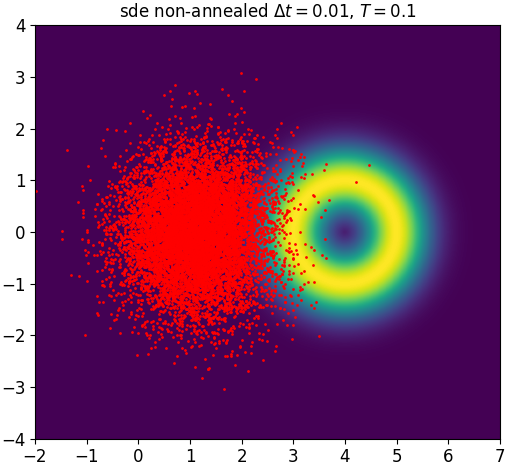}
    \includegraphics[width=0.32\linewidth]{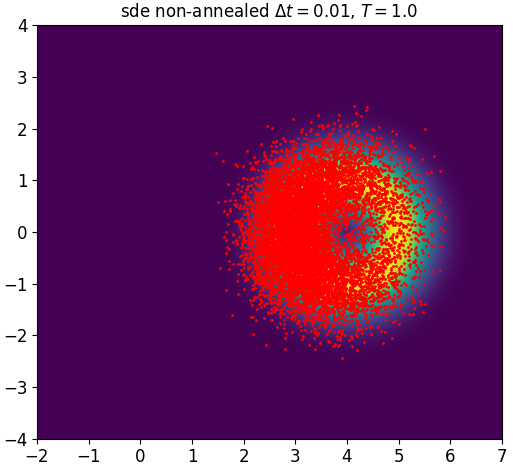}
    \includegraphics[width=0.32\linewidth]{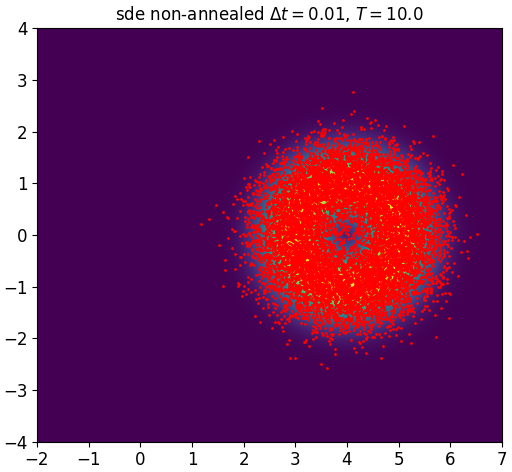}
    \caption{Experiment \ref{exp: noisy circle}, noisy circle. SBTM (ours, top) leaves the vacuum region empty, while SDE (bottom) fills it, demonstrating the effect of determinism.}
    \label{fig: noisy circle density}
\end{figure}

In this example we compare the SDE and SBTM samples from the ``noisy circle'' density $\pi \propto \exp\left(-\frac{(\|x - (4,0)\|-1)^2}{0.08}\right).$ While this example is non-log-concave, unlike other non-log-concave examples, there is no energy barrier between modes, and vanilla Wasserstein Gradient Flow converges quickly. We use 10,000 particles and time step $0.01$. While the SDE samples fill the vacuum region in Figure \ref{fig: noisy circle density} around the center due to Brownian noise, the SBTM samples leave it blank due to determinism.

\subsection{Gaussian mixture with dilation annealing}\label{exp: Gaussian mixture with dilation annealing}
\begin{figure}[htp!]
\centering
\vspace{-1.0em}
{\footnotesize
\includegraphics[width=\linewidth]{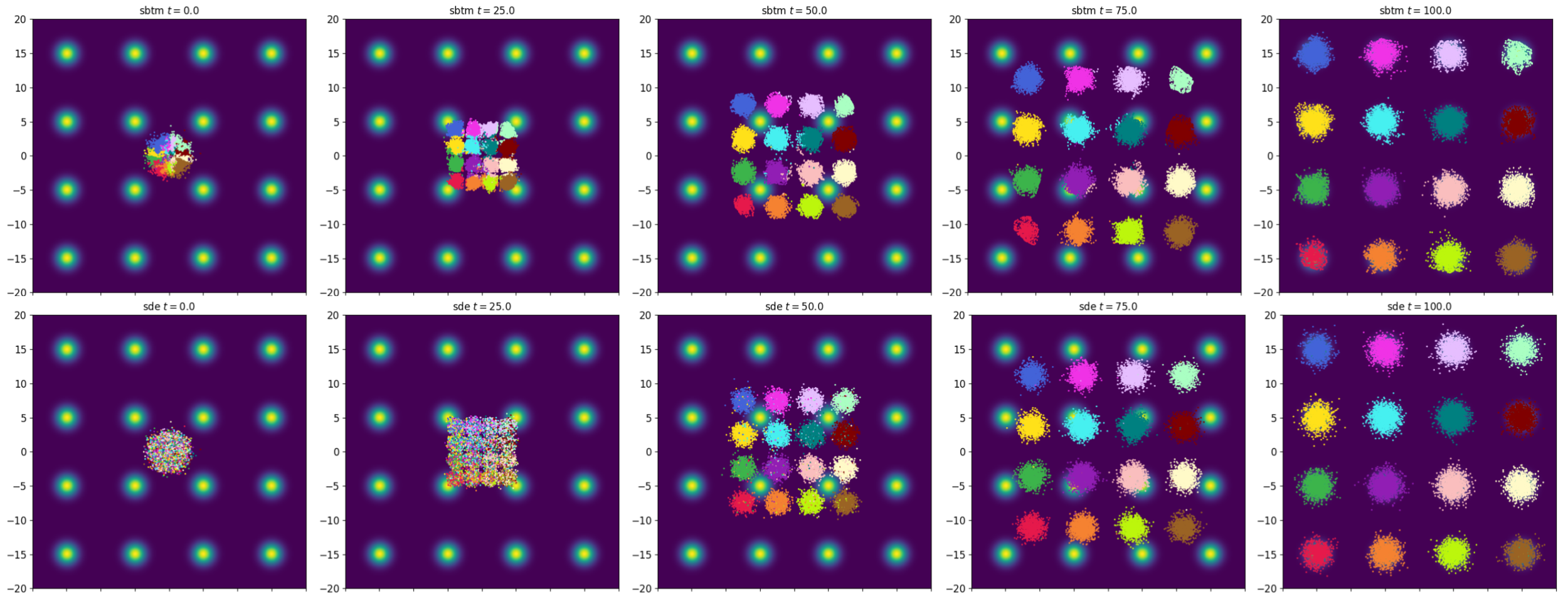}
\caption{Experiment \ref{exp: Gaussian mixture with dilation annealing}, well-separated 2D Gaussian mixture. Scatter plot over time. Top: SBTM (ours) with the dilation annealing. Bottom: SDE with the dilation annealing \cite{chehab2024practical}. SBTM separates into modes early on, compared to the SDE.}
\label{fig: gaussians far 2d heatmap}
}
\end{figure}

\begin{figure}[h]
    \centering
    \includegraphics[width=0.8\linewidth]{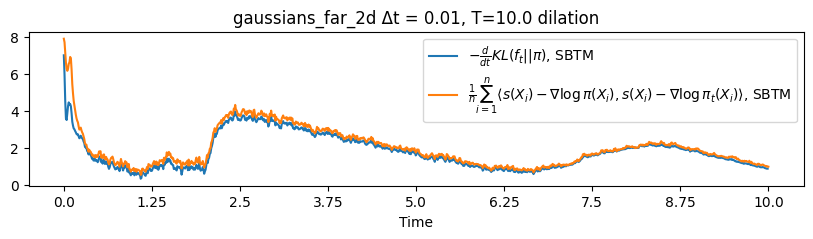}
    \caption{Experiment \ref{exp: Gaussian mixture with dilation annealing}, well-separated 2D Gaussian mixture. The estimate in \eqref{eqn: entropy dissipation in annealed dynamics, assuming zero loss} holds empirically, indicating small score-matching loss and good score approximation in annealed dynamics.}
    \label{fig: gaussians far 2d annealed entropy dissipation}
\end{figure}
To sample from a Gaussian mixture with 16 well-spaced modes we employ the annealing schedule $\pi_t(x) = \pi\left(\frac{T}{t}x\right)$ from \cite{chehab2024practical}. This is a challenging example due to the extreme non-log-concavity of the target. Figure \ref{fig: gaussians far 2d heatmap} shows the densities at different time points demonstrating early mode splitting due to deterministic dynamics. Figure \ref{fig: gaussians far 2d annealed entropy dissipation} demonstrates that the score is learned well by showing that estimate \eqref{eqn: entropy dissipation in annealed dynamics, assuming zero loss} holds almost perfectly. With the large sample size of 20,000 particles and 10,000 time steps this experiment took only 90 minutes on a single TITAN X GPU, leveraging the linear scaling of memory and time complexity of SBTM. For comparison, we were unable to use SVGD with 20,000 particles due to its $O(n^2)$ memory scaling.

\subsection{High-dimension experiments}\label{exp: High-dimension experiments}

To evaluate the scaling of our method to high-dimensional settings, we apply SBTM to produce MNIST digits \cite{lecun1998gradient}, which corresponds to sampling a $784$-dimensional distribution. We purposefully sample directly in pixel space, without using a VAE encoder, to maintain the high dimensionality of the data. To obtain $\score \pi$ we train a U-Net with score-matching. Figure \ref{fig: mnist digits batch 64} shows that SBTM produces high quality results, comparable to the SDE sample. Figure \ref{fig:sbtm_diverse_sample} shows that that SBTM maintains exploration in high dimensional spaces. The amount of training controls the strength of interaction between particles. With sufficient training, particles repel from each other, providing deterministic exploration. To confirm that $\nabla \log f_t$ is being meaningfully learned, we plot the cosine similarity to $\score \pi$ across time in Figure \ref{fig: cosine similarity}. More training epochs per step leads to faster learning of the target score $\score \pi$.

\begin{figure}[htp!]
    \centering
    \includegraphics[width=\linewidth]{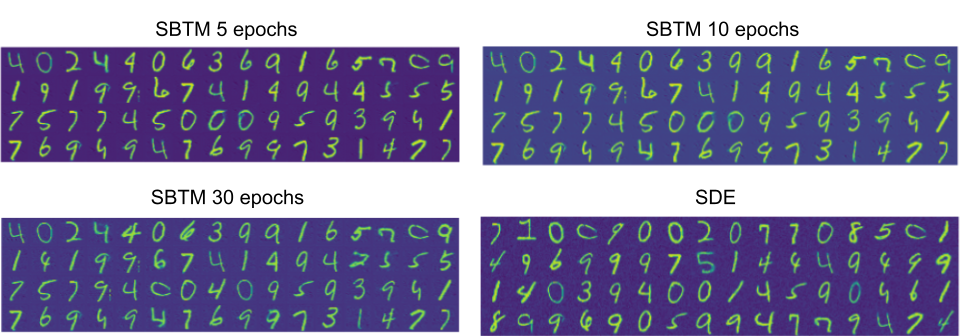}
    \caption{Experiment \ref{exp: High-dimension experiments}, high-dimensional. Sampled MNIST digits using $f_0 = \N(0,1)$, sample size $64$, time step $10^{-3}$, $30$ steps, variable amount of intermediate training. Digits generated by SBTM are competitive in visual quality.}
    \label{fig: mnist digits batch 64}
\end{figure}
\begin{figure}[htp!]
    \centering
    \includegraphics[width=1\linewidth]{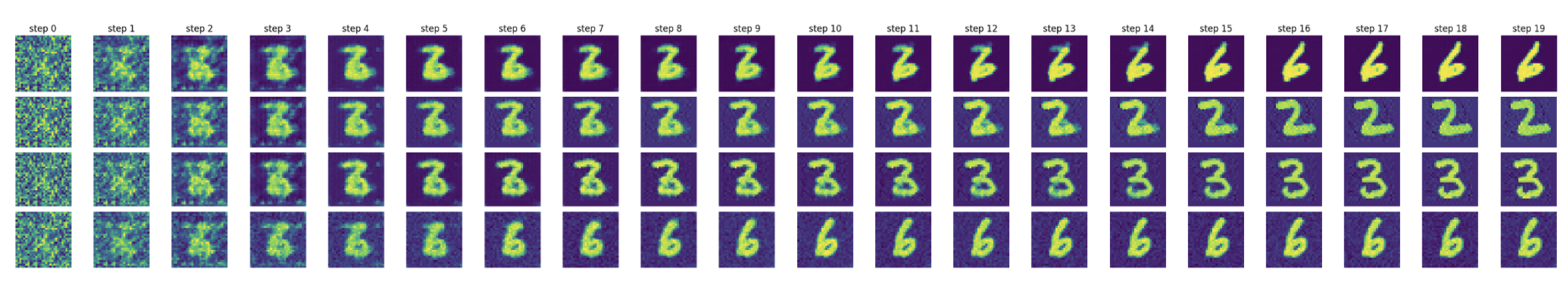}
    \caption{Experiment \ref{exp: High-dimension experiments}, high-dimensional. Starting from the same initial point, SBTM produces distinct sample trajectories depending on the training schedule. The amount of training controls the strength of interaction between particles. Top to bottom: SBTM without training (equivalent to gradient ascent on $\nabla \log \pi$), SBTM with small amount training, SBTM with large amount training, and finally the SDE (Langevin).}
    \label{fig:sbtm_diverse_sample}
\end{figure}
\begin{figure}[htp!]
    \centering
    \includegraphics[width=0.7\linewidth]{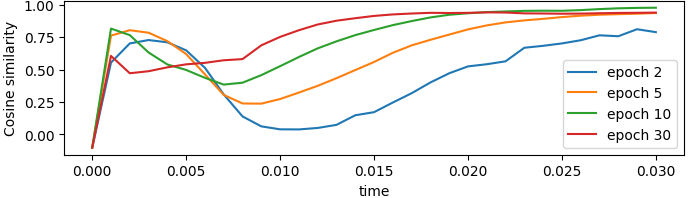}
    \caption{Experiment \ref{exp: High-dimension experiments}, high-dimensional. Cosine similarity between the true score $\nabla \log \pi$ and learned score $\nabla \log f_t$ over simulation time. Even with very little training the model learns the score well. Different lines use different numbers of epochs per time-step, effectively changing $\eta$ in \eqref{thm: main result}.}
    \label{fig: cosine similarity}
\end{figure}

\section{Conclusion}
\label{sec:conclusion}

In this work we use tools and intuition from diffusion generative modeling to tackle the harder problem of sampling $\pi$ given only access to $\score \pi$ but not samples from $\pi$. Our method allows for deterministic sampling with smooth trajectories and provably optimal rate of entropy dissipation. Additionally, access to the learned score $\score f_t$ allows for the computation of relative Fisher info to estimate convergence, making the dynamics more interpretable. Our method integrates well with annealed dynamics to sample from challenging non log-concave densities. Finally, our method scales well both in the sample size, with $O(n)$ complexity, and in dimension, as demonstrated in the 784-dimensional example.

\section*{Acknowledgments}
We thank Bamdad Hosseini and Omar Chehab for helpful discussions about this work. JH's work was partially supported by DOE grant DE-SC0023164, NSF grants DMS-2409858 and IIS-2433957, and DoD MURI grant FA9550-24-1-0254.

\printbibliography

\end{document}